\def\year{2015}
\documentclass[letterpaper]{article}
\usepackage{aaai}
\usepackage{times}
\usepackage{helvet}
\usepackage{courier}

\usepackage{url}

\usepackage{algorithm}
\usepackage{algorithmic}
\usepackage{comment}
\usepackage{amsfonts}
\usepackage{graphicx}
\usepackage{amssymb, amsmath, multirow, paralist, subfigure}
\usepackage{amsthm,bm}
\def \x {\mathbf{x}}
\def \R {\mathbb{R}}

\def \Q {\mathbf{Q}}

\def \Ocal {\mathcal{O}}

\def \q {\mathbf{q}}
\def \s {\Sigma}

\def \e {\varepsilon}
\def \r {\mathcal{R}}
\def \q {\mathbf{q}}
\newtheorem{thm}{Theorem}
\newtheorem{lem}{Lemma}
\newtheorem{cor}[thm]{Corollary}

\begin{document}
%
\title{Similarity Learning via Adaptive Regression and \\Its Application to Image Retrieval}

\author{Qi Qian$^1$, Inci M. Baytas$^1$, Rong Jin$^1$, Anil Jain$^1$ and Shenghuo Zhu$^2$\\
$^1$Department of Computer Science and Engineering, Michigan State University, East Lansing, MI 48824\\
$^2$Alibaba Group\\
}
\maketitle
\begin{abstract}
\begin{quote}
We study the problem of similarity learning and its application to image retrieval with large-scale data. The similarity between pairs of images can be measured by the distances between their high dimensional representations, and the problem of learning the appropriate similarity is often addressed by distance metric learning. However, distance metric learning requires the learned metric to be a PSD matrix, which is computational expensive and not necessary for retrieval ranking problem. On the other hand, the bilinear model is shown to be more flexible for large-scale image retrieval task, hence, we adopt it to learn a matrix for estimating pairwise similarities under the regression framework. By adaptively updating the target matrix in regression, we can mimic the hinge loss, which is more appropriate for similarity learning problem. Although the regression problem can have the closed-form solution, the computational cost can be very expensive. The computational challenges come from two aspects: the number of images can be very large and image features have high dimensionality. We address the first challenge by compressing the data by a randomized algorithm with the theoretical guarantee. For the high dimensional issue, we address it by taking low rank assumption and applying alternating method to obtain the partial matrix, which has a global optimal solution. Empirical studies on real world image datasets (i.e., {\it Caltech} and {\it ImageNet}) demonstrate the effectiveness and efficiency of the proposed method.
\end{quote}
\end{abstract}

\section{Introduction}

Learning pairwise similarity is important for machine learning tasks, e.g., classification~\cite{GuoY14,BelletHS12}, clustering~\cite{YiJJJY12,XingNJR02}, ranking~\cite{chechik2010,LimL14}, etc. Distance metric learning (DML), which aims to learn the appropriate distance (i.e., similarity) for any given pair of instances, has been studied for decades and various methods have been proposed to estimate the similarity. Given the learned metric $M$, most of the DML methods compute the similarity in the form of Mahalanobis distance~\cite{mahalanobis1936}: $\mathrm{dis}_M(\x_i,\x_j) = (\x_i-\x_j)^\top M (\x_i-\x_j)$, where $M$ has to be the positive semi-definite (PSD) matrix. The PSD constraint guarantees that the similarity value is nonnegative and symmetric, which is the requirement for distance based classification (e.g., $k$-nearest neighbor) and clustering (e.g., $k$-means) methods.

For certain applications, however, the properties of nonnegativity and symmetry are not necessary. Ranking tasks, for example, only require to output a list of examples according to the query and hence the similarity value could be arbitrary (non metric). Additionally, the task of image retrieval is known to be non-symmetric according to human judgement~\cite{amos77}. In this scenario, the similarity can be measured by the similarity function in bilinear form~\cite{GuoY14,chechik2010}: $\mathrm{Sim}_M(\x_i,\x_j) = \x_i^\top M \x_j$, which computes the inner product between two examples. Without the PSD constraint as required in DML, the matrix $M$ in the bilinear model can be non-symmetric and the cost of optimization decreases from $\Ocal(d^3)$ (i.e., PSD projection) to $\Ocal(d^2)$, where $d$ is the dimensionality of data. We will adopt the bilinear model in this work.

Given the bilinear model, an appropriate similarity function can be learned over pairwise constraints. There are two challenges when applying it directly to large-scale image data. First,
the number of constraints is quadratic in the number of images (i.e., $n$). Although most of existing DML methods apply stochastic strategy to access one pair at each iteration for efficiency, it is impossible for them to utilize all the pairs, when $n$ is large (e.g., the cost is at least $\Ocal(n^2d^2)$ and up to $\Ocal(n^3d^3)$), which introduces the risk of the suboptimal solution. Second, images tend to have high dimensional features to capture the image content. Since the size of $M$ is $d\times d$, it becomes difficult to store it when $d$ is very large. Further, with large number of parameters, it is easy to encounter the overfitting issue.

In this paper, we attempt to learn a good similarity function by solving a matrix regression problem~\cite{YiJJJY12}. Unlike the conventional matrix regression, the proposed method adaptively update the target metric at each iteration to mimic the hinge loss, which is more appropriate for real world applications. Besides, different from stochastic methods, the approach is deterministic and has the closed-form solution, which is very efficient for optimization and can obtain information from all pairs of data. To adapt it for large-scale high dimensional image retrieval, we further explore the randomized algorithm for data compressing and the low rank assumption. Our contributions can be summarised as follows.
\begin{itemize}
\item We develop an adaptive strategy for updating the target matrix in regression, which is more flexible and appropriate for learning a similarity function with pairwise constraints.
\item To further improve the scalability for large-scale problem, we develop a randomized algorithm to compress the data without sacrificing the performance. Our theoretical analysis shows that if the dataset matrix is of low rank with rank $r_\mathcal{D}\ll \min\{n,d\}$, then only $\Ocal(r_\mathcal{D}\log(r_\mathcal{D}))$ examples are sufficient to obtain a good similarity function, which is much better over the previous result (i.e., $\Ocal(d^2)$~\cite{DrineasMM06}). It is important when the data matrix or the target matrix cannot be loaded into the memory.
\item To alleviate the high dimensional challenge, we take the low rank assumption~\cite{weinberger2009} that the rank of the optimal $M$ is significantly smaller than the dimensionality of data (i.e., $r\ll d$). Then, the low rank component of $M$ can be optimized by an alternating method efficiently. Although the corresponding optimization problem is non-convex, the recent improvement in the alternating method shows that the solution is globally optimal with the geometric convergence rate~\cite{prateek12}.
\end{itemize}

We conduct empirical studies on real world image datasets; comparison with the state-of-the-art methods verify both the effectiveness and efficiency of the proposed method.

\section{Related Work}
\label{sec:related}

\subsection{Distance Metric Learning}

Many methods have been developed to estimate pairwise similarity. Some of them can be categorized as distance metric learning, where a distance metric is learned to increase the pairwise distance involving examples from different classes and reduce the pairwise distance involving examples from the same class. The representative methods include Xing's method~\cite{XingNJR02}, POLA~\cite{Shai04}, ITML~\cite{DavisKJSD07}, LMNN~\cite{weinberger2009}, and FRML~\cite{LimL14}. ITML learns a metric according to the pairwise constraints, where the distance between pairs from the same class should be smaller than a pre-defined threshold and the distance between pairs from different classes should be larger than a second threshold. LMNN is developed with triplet constraints and a metric is learned to make sure that pairs from the same classes are separated from the pairs from different classes with a large margin. PSD constraint is adopted by these methods to make sure that the learned similarity matrix is symmetric and nonnegative. All of these DML methods have demonstrated some success in real applications. However, PSD projection is time consuming while its property may be not necessary in certain applications.

\subsection{Similarity Learning with Bilinear Model}
Besides these DML methods, the bilinear model is developed for learning a similarity without the PSD constraint. Different from the similarity defined on the distance of two examples, the bilinear model computes the inner product of two examples by the learned similarity function. OASIS~\cite{chechik2010} successfully applied it to the ranking problem, which requires that the more relevant examples have larger similarity and a hinge loss is adopted for a safety margin. The corresponding optimization problem is solved by an online learning method for large-scale image retrieval task. Compared with DML, the learned matrix is not symmetric, which offers more flexibility in real world applications. Although OASIS could handle millions of images, the computational cost is still $\Ocal(d^2)$ at each iteration. Moreover, online learning cannot enumerate all triplet constraints due to its huge number (i.e., $\Ocal(n^3)$), which may lead to a suboptimal solution. In this paper, we propose a matrix regression model for large-scale high dimensional image retrieval task.

\section{Our Approach}
\label{sec:method}
\subsection{Adaptive Regression for Similarity Learning}
Given a set of training images $X\in\R^{d\times n}$, similarity learning with bilinear model aims to learn a good matrix $M$ for estimating pairwise similarity as
\[\mathrm{Sim}_M(\x_i,\x_j) = \x_i^\top M \x_j\]

With the target matrix $Y\in \R^{n\times n}$, $M$ can be learned via solving a matrix regression problem
\begin{eqnarray}\label{eqn:ori}
\min_{M\in \R^{d\times d}} \|X^\top M X-Y\|_F^2
\end{eqnarray}
where $Y$ can be generated by the label information and defined as
\begin{eqnarray}\label{eqn:label}
Y_{i,j} = \left\{\begin{array}{r@{\quad:\quad}l} 1 & \mathrm{label}(\x_i)=\mathrm{label}(\x_j)\\ 0 & otherwise\end{array} \right.
\end{eqnarray}
This formulation is popular for similarity learning and widely used in different tasks~\cite{YiJJJY12,FengJJ13}. It has a closed-form solution for the matrix $M$:
\begin{eqnarray}\label{eqn:solution}
M_* = (X^\dagger)^\top Y X^\dagger
\end{eqnarray}
where $X^\dagger$ is the pseudo inverse of $X$.

However, the constraints in the similarity learning appear as
\begin{eqnarray*}
\x_i^\top M \x_j \geq \delta_1;\quad \x_i^\top M \x_k \leq \delta_2
\end{eqnarray*}
where $\x_i$ and $\x_j$ are from the same class and the label of $\x_k$ is different. $\delta_1$ and $\delta_2$ are pre-defined thresholds.
The corresponding definition for the target matrix should be
\[Y_{i,j} =\left\{\begin{array}{r@{\quad:\quad}l} \geq \delta_1 & \mathrm{label}(\x_i)=\mathrm{label}(\x_j)\\ \leq \delta_2 & otherwise\end{array} \right. \]
and the corresponding optimiztion problem with the squared hinge loss is
\begin{eqnarray}\label{eqn:hinge}
\min_{M\in \R^{d\times d}} \sum_{i,j} [\delta_1-\x_i^\top M\x_j]_+^2 +\sum_{i,k}[\x_i^\top M\x_k-\delta_2]_+^2
\end{eqnarray}
The problem in Eqn.~\ref{eqn:hinge} can be solved by SGD but will suffer from the problem of suboptimal solution.

Therefore, we take an alternating update strategy, which puts less penalty on the pairs that have been separated well and focuses on the hard pairs of examples, for the target matrix to mimic the hinge loss as summarized in Alg.~\ref{alg:1}. The proposed method, on the one hand, can approximates the hinge loss, which is more appropriate for real world application. On the other hand, it enjoys the light computational cost from the closed-form solution of the conventional regression task.

\begin{algorithm}[h]
\caption{Adaptive Regression for Similarity Learning}
\begin{algorithmic}[1]
\STATE {\bf Input:} Dataset $X\in \R^{d\times n}$, $\#$ Iterations $T$, thresholds $\delta_1$ and $\delta_2$
\STATE Initialize $M_0$ as an identity matrix
\FOR{$k = 1,\cdots,T$}
\STATE $\hat{Y} = X^\top M_{k-1} X$
\STATE $Y = \left\{\begin{array}{r@{\quad:\quad}l} \max\{\hat{Y}_{i,j}, \delta_1\} & \mathrm{label}(\x_i)=\mathrm{label}(\x_j)\\ \min\{\hat{Y}_{i,j},\delta_2\} & otherwise\end{array} \right.$
\STATE Learn $M_{k} = \min_{M\in \R^{d\times d}}\|X^\top M X-Y\|_F^2$
\ENDFOR
\RETURN $M_T$
\end{algorithmic}\label{alg:1}
\end{algorithm}

\subsection{Large-scale Challenge}

For obtaining the closed-form solution at each iteration, it has to compute the pseudo inverse of $X$, which costs $\Ocal(\min\{n^2d,d^2n\})$. Besides, solving the regression problem in Alg.~\ref{alg:1} requires the appearance of the $n\times n$ target matrix $Y$. When the number of examples is extremely large, this is not affordable to maintain $X$ and $Y$ in the memory. Therefore, we try to reduce the number of examples, which is equivalent to the number of constraints, at each iteration by developing a randomized algorithm.

Given a general regression problem
\begin{eqnarray}\label{eqn:iterk}
\min_{M} \|AMB^\top - Y\|_F^2
\end{eqnarray}
where $A=X^\top$ and $B = X$ in our application, we will solve a problem with compressed examples instead
\begin{eqnarray}\label{eqn:rand}
\min_{M} \|S_1^\top AMB^\top S_2 - S_1^\top YS_2\|_F^2
\end{eqnarray}
where $S_1, S_2\in \R^{n\times m}$ are two random matrices and the resulting target matrix is only $m\times m$. The following theorem shows that the solution from problem (\ref{eqn:rand}) is close to the solution in the original problem in (\ref{eqn:iterk}).

\begin{thm}\label{thm:random}
Let $M$ be the optimal solution of problem (\ref{eqn:iterk}) and $\hat{M}$ be the solution for problem (\ref{eqn:rand}). Let $\r$ be the residual from $M$ in problem (\ref{eqn:iterk}). Let $\delta_{\min}^A$ and $\delta_{\max}^A$ represent the minimal and maximal singular value of the matrix $A$, respectively. $S_1,S_2\in\R^{n\times m}$ are two normalized Gaussian random matrices and the rank of the data matrix $A$ and $B$ is $r_\mathcal{D}$.  Then, with a probability $1-\delta$, we have
\begin{eqnarray*}
\|M-\hat{M}\|_2\leq \frac{\|\r\|_2}{\delta_{\min}^A \delta_{\min}^B}\left(\frac{\e^2(1+\e)}{1-\e}\right.\\
\left.+2\frac{\e(1+\e)^{3/2}}{\sqrt{1-\e}}+2\sqrt{2\e}(1+\e)^{3/2})\right)
\end{eqnarray*}
provided
\[m \geq \frac{(r_\mathcal{D}+1)\log(48r_\mathcal{D}/\delta)}{c\e^2}\]
Furthermore, if we take a similar assumption as in~\cite{DrineasMM06} that $\r$ is bounded by the estimate of the optimal solution by a factor $\gamma\in(0,1)$
\[\|Y-AMB^\top\|_2\leq \gamma \|AMB^\top\|_2\]
we have
\begin{eqnarray*}
&&\|M-\hat{M}\|_2\leq \frac{\gamma\delta_{\max}^A \delta_{\max}^B\|M\|_2}{\delta_{\min}^A\delta_{\min}^B}(\frac{\e^2(1+\e)}{1-\e}\\
&&+2\frac{\e(1+\e)^{3/2}}{\sqrt{1-\e}}+2\sqrt{2\e}(1+\e)^{3/2})
\end{eqnarray*}
\end{thm}

\paragraph{Remark 2} We only require $\Ocal(r_\mathcal{D}\log(r_\mathcal{D}))$ random projections which is significantly smaller than the result in~\cite{DrineasMM06}, where the sampling number is at least $\Ocal(d^2)$. Compared with the cost of obtaining $X^\dagger$ (i.e., $\Ocal(\min\{n^2d,nd^2\})$), the cost of computing $(XS)^\dagger$ is only $\Ocal(r_\mathcal{D}^2d+r_\mathcal{D}nd$), which is feasible for large-scale problems. 

\paragraph{Remark 3} Although we use Gaussian random matrices, the essential requirement for the random matrix is $E[SS^\top]=I$. Therefore, it can be a matrix with each column randomly selected from $\{\sqrt{n/m}\mathbf{e}_1,\cdots,\sqrt{n/m}\mathbf{e}_n\}$, which is equivalent to uniformly sampled columns of the data matrix. Compared with random projection, it saves the computational cost of compressing the dataset (i.e. $\Ocal(r_\mathcal{D}nd)$ in random projection) and the total cost of the algorithm is only $\Ocal(r_\mathcal{D}^2d+rr_\mathcal{D}d)$. Besides, we can show that
\begin{lem}\label{lemma:sampling}
Let $S \in \R^{n\times m}$ be a random sampling matrix, where each column is uniformly sampled from $\{\sqrt{n/m}\mathbf{e}_1,\cdots,\sqrt{n/m}\mathbf{e}_n\}$, with a probability $1-\delta$. We have
\[\|SS^\top -I\|_2\leq \e\]
provided
\[m\geq \frac{(n-1)\log{n/\delta}}{c\e^2}\]
where constant $c$ is at least $1/3$.
\end{lem}

Alg.~\ref{alg:2} summarizes the randomized variant of the proposed method.

\begin{algorithm}[h]
\caption{Randomized Regression for Similarity Learning}
\begin{algorithmic}[1]
\STATE {\bf Input:} Dataset $X\in \R^{d\times n}$, $\#$ Iterations $T$, $m$
\STATE Initialize $M_0$ as an identity matrix
\FOR{$k = 1,\cdots,T$}
\STATE Generate two random matrices $S_1, S_2\in \R^{n\times m}$
\STATE Compute the pairs selected by the random matrices $\hat{Y}_S = X^\top M_{k-1} X$
\STATE $Y_S = \left\{\begin{array}{r@{\quad:\quad}l} \max\{\hat{Y}_{i,j}, \delta_1\} & \mathrm{label}(\x_i)=\mathrm{label}(\x_j)\\ \min\{\hat{Y}_{i,j},\delta_2\} & otherwise\end{array} \right.$
\STATE Learn $M_{k} = \min_{M\in \R^{d\times d}}\|S_1^\top X^\top M XS_2-Y_S\|_F^2$
\ENDFOR
\RETURN $M=L_TR_T^\top$
\end{algorithmic}\label{alg:2}
\end{algorithm}

\subsection{High Dimensional Challenge}

To avoid overfitting and alleviate the storage challenge, the learned matrix is constrained to be of low rank and the corresponding problem is
\[\min_{M\in \R^{d\times d},rank(M)=r} \|X^\top M X-Y\|_F^2\]
where $r$ is the rank of the learned matrix and $r\ll d$.

Although this problem has the closed-form solution~\cite{YuS11}, the solution is only available for Frobenius norm and the computational cost is expensive (i.e., $\Ocal(dn^2+n^3)$).

Inspired by the idea of alternating optimization~\cite{prateek12,Netrapalli0S13}, we can decompose $M$ as $M = LR^\top$, where $L,R\in\R^{d\times r}$. Then the original problem is equivalent to
\begin{eqnarray}\label{eqn:alter}
\min_{L,R\in \R^{d\times r}} \|X^\top LR^\top X-Y\|_F^2
\end{eqnarray}
Note that $L$ and $R$ can be different since $M$ is not a PSD matrix.
The problem in (\ref{eqn:alter}) can be solved via the alternating method. Since the component ``$X^\dagger$'' can be reused at each iteration, it is just computed once at the beginning of the algorithm and the cost of each iteration is only $\Ocal(ndr+n^2r)$, which is affordable for high dimensional data. 

Alg.~\ref{alg:3} shows the algorithm with data compressing. Each subproblem (i.e., Steps \ref{step:1} and \ref{step:2}) has a closed-form solution and only a partial random projection is needed since the size of $X^\top R$ or $X^\top L$ is only $n\times r$ and the cost of computing pseudo inverse is light.

\begin{algorithm}[h]
\caption{Alternating Regression for Similarity Learning}
\begin{algorithmic}[1]
\STATE {\bf Input:} Dataset $X\in \R^{d\times n}$, $\#$ Iterations $T$, $m$
\STATE Initialize $L_0$ and $R_0$ using a randomly sampled subset
\FOR{$k = 1,\cdots,T$}
\STATE Generate a random matrix $S\in \R^{n\times m}$
\STATE Compute the corresponding target matrix $Y_S$
\STATE Learn $L_{k} = \min_{L\in \R^{d\times r}}\|S^\top X^\top LR_{k-1}^\top X-Y_S\|_F^2$ \label{step:1}
\STATE Generate a random matrix $S\in \R^{n\times m}$
\STATE Compute the corresponding target matrix $Y_S$
\STATE Learn $R_{k} = \min_{R\in \R^{d\times r}}\|X^\top L_kR^\top XS-Y_S\|_F^2$     \label{step:2}
\ENDFOR
\RETURN $M=L_TR_T^\top$
\end{algorithmic}\label{alg:3}
\end{algorithm}

\begin{cor}
Assume that the rank of the optimal solution $M_*$ for the Problem (\ref{eqn:ori}) is $r$ and the linear measurement $\mathcal{A}(M) = X^\top MX$ satisfies restricted isometry property (RIP)~\cite{RechtFP10}. When $n<d$, we have that the solution from Alg.~\ref{alg:1} is globally optimal with a geometric convergence rate
\[\|M_*-L_TR_T\|_F\leq e^{-T/2}\|M_*\|_F\]
\end{cor}
\begin{proof}
It is directly from the Theorem 2.2 in \cite{prateek12}.
\end{proof}
\section{Theoretical Analysis}

\subsection{Proof of Theorem~\ref{thm:random}}
\label{sec:analysis}

The key of our proof is from the following corollary.
\begin{cor} \label{cor:gaussian}
~\cite{Zhang13} Let $S \in \R^{r\times m}$ be a standard Gaussian random matrix. Then, for any $0<\e\leq 1/2$, with a probability $1 - \delta$, we have
\[
    \left\|\frac{1}{m}SS^{\top} - I \right\|_2 \leq \e
\]
provided
\[
    m \geq \frac{(r+1)\log(2r/\delta)}{c\e^2}
\]
where constant $c$ is at least $1/4$.
\end{cor}

The main idea of the proof is similar to that in~\cite{DrineasMM06}. Before the proof, we first give some useful lemmas.
\begin{lem}\label{lemma:rank}
Given a rank $r$ matrix $A = U_A\s_AV_A^\top\in \R^{n\times d}$ and a normalized Gaussian random matrix $S\in \R^{n\times m}$, which is normalized by $\sqrt{m}$, with a probability $1-\delta$, we have
\[|1-\delta_i^2(S^\top U_A)|\leq \e\]
and
\[rank(S^\top U_A) = rank(U_A)\]
where $m = \Ocal(r\log(r))$ and $\delta_i(S^\top U_A)$ is the $i$-th singular value of $S^\top U$.
\end{lem}

\begin{lem}\label{lemma:d2t}
Let $A$ and $S$ are the matrices defined in Lemma~\ref{lemma:rank}, with a probability $1-\delta$, we have
\[
(S^\top A)^\dagger = V_A\s_A^{-1} (S^\top U_A)^\dagger
\]
\end{lem}

\begin{lem}\label{lemma:omega}
Let $\Omega = (S^\top U_A)^\dagger S^\top - (S^\top U_A)^\top S^\top$, where $U_A$ and $S$ are the matrices in Lemma~\ref{lemma:rank} and $X^\dagger$ is the pseudo inverse of matrix $X$, with a probability $1-2\delta$, we have
\[\|\Omega\|_2\leq\frac{\e\sqrt{1+\e}}{\sqrt{1-\e}}\]
\end{lem}

We skip the proofs for these lemmas since the proof is similar as in \cite{DrineasMM06}. Now, we show the key steps of the proof for Theorem~\ref{thm:random} due to the space limitation.


We try to bound the difference between $M$ and $\hat{M}$ as
\begin{eqnarray*}
\|M-\hat{M}\|_2 = \|A^\dagger Y (B^\top)^\dagger - (S_1^\top A)^\dagger S_1^\top Y S_2 (B^\top S_2)^\dagger\|_2
\end{eqnarray*}
where $M$ and $\hat{M}$ are the solutions for problems (\ref{eqn:iterk}) and (\ref{eqn:rand}), respectively.
According to Lemma~\ref{lemma:d2t}, with a probability $1-2\delta$, we have
\begin{eqnarray*}
\|M-\hat{M}\|_2\leq \frac{1}{\delta_{\min}^A\delta_{\min}^B}\|(S_1^\top U_A)^\dagger S_1^\top \r S_2 (U_B^\top S_2)^\dagger \|_2
\end{eqnarray*}
where the residual $\r = Y-AMB^\top = U_A^\bot U_A^{\bot\top} Y+U_AU_A^\top Y U_B^\bot U_B^{\bot\top}$

Let $\Omega_A = (S_1^\top U_A)^\dagger S_1^\top-(S_1^\top U_A)^\top S_1^\top$ and $\Omega_B = S_2(B^\top S_2)^\dagger - S_2(B^\top S_2)^\top$.
Then
\begin{eqnarray}\label{eqn:2norm}
&&\|M-\hat{M}\|_2 \leq\nonumber\\
&&\frac{1}{\delta_{\min}^A\delta_{\min}^B}\|(\Omega_A+U_A^\top S_1 S_1^\top) \r (\Omega_B+S_2S_2^\top U_B) \|_2 \nonumber
\end{eqnarray}

First, we can proof that with a probability $1-2\delta$, it is
\begin{eqnarray}\label{eqn:comp1}
\|U_A^\top SS^\top\|_2 \leq 1+\e
\end{eqnarray}
Using the similar procedure, with a probability $1-3\delta$, we have
\begin{eqnarray}\label{eqn:comp2}
\|U_AU_A^\top - U_AU_A^\top SS^\top\|_2\leq \sqrt{\e^2+2\e}
\end{eqnarray}

Since $\r$ has two components, we investigate the pattern as
\begin{eqnarray}\label{eqn:comp3}
&&\|U_A^\top S_1 S_1^\top U_A^\bot U_A^{\bot\top} Y\|_2 = \|U_AU_A^\top S_1 S_1^\top U_A^\bot U_A^{\bot\top} Y\|_2\nonumber\\
&&=\|U_AU_A^\top U_A^\bot U_A^{\bot\top} Y-U_AU_A^\top S_1 S_1^\top U_A^\bot U_A^{\bot\top} Y\|_2\nonumber\\
&&\leq\|U_AU_A^\top - U_AU_A^\top S_1 S_1^\top\|_2\|U_A^\bot U_A^{\bot\top} Y\|_2
\end{eqnarray}

By taking Eqn.~\ref{eqn:comp1}-\ref{eqn:comp3} back to the Eqn.~\ref{eqn:2norm}, and applying union bound, with a probability $1-24\delta$ we have
\begin{eqnarray*}
&&\|M-\hat{M}\|_2\leq \frac{\|\r\|_2}{\delta_{\min}^A\delta_{\min}^B}(\frac{\e^2(1+\e)}{1-\e}+2\frac{\e(1+\e)^{3/2}}{\sqrt{1-\e}})\\
&&+\frac{(1+\e)\sqrt{\e^2+2\e}}{\delta_{\min}^A\delta_{\min}^B}(\| U_A^\bot U_A^{\bot\top}Y\|_2\\
&&+\|U_AU_A^\top Y U_B^\bot U_B^{\bot\top}\|_2)
\end{eqnarray*}
We finish the proof with
\begin{eqnarray}\label{eqn:residual}
2\|\r\|_2> \| U_A^\bot U_A^{\bot\top}Y\|_2+\|U_AU_A^\top Y U_B^\bot U_B^{\bot\top}\|_2
\end{eqnarray}

Furthermore, if we take a similar assumption as in~\cite{DrineasMM06}
\[\|\r\|_2 = \|Y - U_AU_A^\top YU_BU_B^\top\|_2\leq \gamma \|U_AU_A^\top YU_BU_B^\top\|_2\]
where $\gamma\in(0,1)$, we have
\begin{eqnarray*}
&&\|M-\hat{M}\|_2\leq \frac{\gamma\delta_{\max}^A \delta_{\max}^B\|M\|_2}{\delta_{\min}^A\delta_{\min}^B}(\frac{\e^2(1+\e)}{1-\e}\\
&&+2\frac{\e(1+\e)^{3/2}}{\sqrt{1-\e}}+2\sqrt{2\e}(1+\e)^{3/2})
\end{eqnarray*}
due to
\[\|M\|_2=\|A^\dagger Y B^{\dagger\top}\|_2\geq \frac{1}{\delta_{\max}^A \delta_{\max}^B}\|U_AU_A^\top YU_BU_B^\top\|_2 \]

\subsection{Proof of Lemma~\ref{lemma:sampling}}
\begin{proof}
We define the random variable $\Q_i$ as
\[\Q_i = n\q_i\q_i^\top-I\]
where $\q_i$ is randomly selected from $\{\mathbf{e}_1,\cdots,\mathbf{e}_n\}$ with probability $1/n$.
It is obvious that $\Q_i$ is self-adjoint, $E[\Q_i]=0$ and $\delta_{\max}\leq n-1$.
Furthermore, we have
\[E[\Q_i^2] = n^2E[\q_i\q_i^\top\q_i\q_i^\top] - 2nE[\q_i\q_i^\top]+I = (n-1)I\]
Therefore, $\sigma^2 = \|\sum_i^m E[\Q_i^2]\|_2 = m(n-1)$.
According to Bernstein's inequality~\cite{Recht11}, we have
\[\Pr\{\delta_{\max} (\sum_{i=1}^m \Q_i)\geq \e\}\leq n \exp(\frac{-\e^2/2}{m(n-1)+(n-1)\e/3})\]
With a similar procedure, we can bound $\delta_{\max} (-\sum_i \Q_i)$.

Combining them together, with a probability $1-\delta$, we have
\[\|SS^\top -I\|_2\leq \e\]
provided \[m\geq \frac{(n-1)\log{n/\delta}}{c\e^2}\]
where each column of $S$ is uniformly sampled from $\{\sqrt{n/m}\mathbf{e}_1,\cdots,\sqrt{n/m}\mathbf{e}_n\}$ and constant $c \geq 1/3$.
\end{proof}

\section{Experiments}
\label{sec:exp}

Four state-of-the-art similarity learning algorithms are included in comparison to verify the effectiveness of the proposed method.
\begin{compactitem}
\item LMNN~\cite{weinberger2009}: DML methods with triplet constraints;
\item FRML~\cite{LimL14}: Efficient DML methods for ranking with mini-batch strategy;
\item OASIS~\cite{chechik2010}: Bilinear model for similarity learning;
\item SLR: Similarity learning with adaptive regression and alternating solver.
\end{compactitem}.

Besides, Euclid is the baseline with Euclidean distance directly. We use the implementations provided by the authors with the recommended parameter settings. We set the number of constraints for FRML and OASIS as $N=10^6$ to fully explore the information contained in the data. Note that this number is still much less than $n^2$. Since LMNN and FRML can optimize the low rank part of the metric (i.e., $L$) directly, we set the rank of the metric as $100$ which is also used by our method. Thresholds $\delta_1$ and $\delta_2$ are set to be $1$ and $0$, respectively, while the number of iterations is set as $T=10$. All experiments are implemented on a server with 64GB memory and $12\times 2.4$GHz CPUs.

We evaluate the learned metric on a ranking task, which is the same as in~\cite{chechik2010}. For the DML methods, the ranking list is given according to the distance to the query in the increasing order while that for the bilinear model is output by the similarity in the descending order. The mean-average-precision (mAP) is used to evaluate the ranking performance. Below is a description of the image datasets used in our experiments.

\begin{table}[!ht]
\centering
\caption{Comparison of mAP($\%$) on {\it Caltech101}.}\label{tab:1}
\begin{tabular}{|c|c|c|c|c|}
\hline
Euclid                 &LMNN            &FRML          &OASIS              &SLR \\\hline
27.7$\pm$0.7         &43.7$\pm$0.9  & 46.6$\pm$0.9&48.7$\pm$0.8&\textbf{55.3$\pm$0.9}   \\\hline
\end{tabular}
\end{table}

\begin{figure*}[!ht]
\centering
\includegraphics[width=6in]{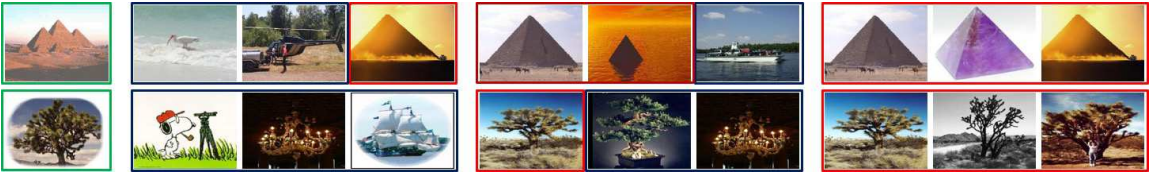}
\caption{Examples of retrieved images. The first column indicates the query images highlighted by green bounding boxes. Columns $2$-$4$ include the most similar images measured by the metric learned by FRML. Columns $5$-$7$ show those by the metric from OASIS. Columns $8$-$10$ are from the metric of SLR. Images in columns $2$-$10$ are highlighted by red bounding boxes when they share the same category as queries, and blue bounding boxes if they are not.}\label{fig:1}
\end{figure*}

\subsection{\it Caltech101}
{\it Caltech101} is a benchmark image dataset which consists of $101$ different objects~\cite{Fei-FeiFP07}. We extract LLC features~\cite{YangYGH09} as representations for each image and reduce the dimension to $1,000$ by PCA to make the comparison with other methods. The final set contains $8,677$ examples with $1,000$ features. We randomly select $70\%$ for training and the rest for testing. The experiments are repeated $5$ times with different splits (of the same proportion between training and test sets) and average results with standard deviation are reported in Table~\ref{tab:1}. 

First, it is obvious that both distance based and similarity based methods outperforms the Euclidean distance. Second, we observe that methods using the bilinear model (i.e., OASIS and SLR) are better than the distance based method. This is because the bilinear model is more flexible for ranking. Second, the proposed method outperforms other methods significantly, which is because SLR can visit all pairs and hence no information is lost. Fig.~\ref{fig:1} shows the examples of retrieved images. The metric learned by SLR can return the right images while the metrics of other methods make certain mistakes. 

To demonstrate the advantage of adaptive regression, we compare SLR to the variant with the fixed target matrix as in Eqn.~\ref{eqn:label}, which is denoted as ``\textbf{SLR-Fixed}''. We also include the whole metric computed from the closed-form solution as in Eqn.~\ref{eqn:solution}, which is denoted as ``\textbf{SLR-Whole}'' for comparison, and the result is summarized in Fig.~\ref{fig:2}. It is observed that SLR is even better than SLR-Fixed and SLR-Whole, which verifies the effectiveness of the adaptive regression. Moreover, we find that SLR-Fixed almost performs the same as SLR-Whole, and it confirms that the alternating method can achieve global optimal solution for the non-convex problem. Finally, SLR and SLR-Fixed converges to a good solution in only $2$ iterations, which is consistent with the analysis of the alternating method.
\begin{figure}[!ht]
\centering
\includegraphics[width=2in]{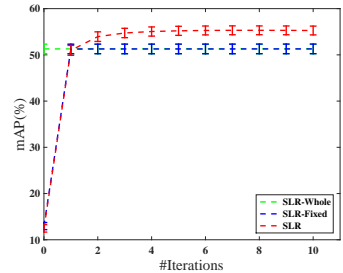}
\caption{Comparison of SLR with the fixed target matrix.}\label{fig:2}
\end{figure}

We further conduct the experiments to illustrate the influence of data compressing. We denote SLR with Gaussian random projection and alternating target matrix (i.e., in Alg.~\ref{alg:3}) as ``\textbf{SLR-G-*}'' and column sampling as ``\textbf{SLR-C-*}''. The number $*$ denotes the dimension after compressing and results are shown in Fig.~\ref{fig:3}. We can find that the performance of SLR-G and SLR-C can approach that of SLR, which verifies our analysis in Theorem~\ref{thm:random}.

\begin{figure}[!ht]
\centering
\includegraphics[width=2in]{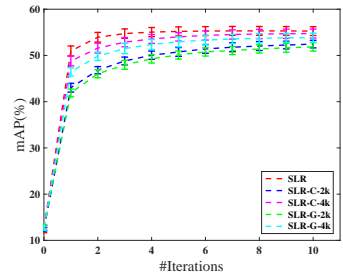}
\caption{Comparison of SLR with data compressing.}\label{fig:3}
\end{figure}

\subsection{{\it $ImageNet50$}}
{\it ImageNet50} consists of 50 randomly selected categories from the ImageNet dataset~\cite{ILSVRCarxiv14}. There are $101,687$ images with $1,000$ features in the dataset and the experimental settings are the same as for the {\it Caltech101} dataset except that we reduce the number of SLR's iterations to $5$ due to the large size of training set. Since the whole target metric is too large, we use the result of SLR-C-20k to stand for that of SLR. Table~\ref{tab:2} summarizes the results from different methods. A similar phenomenon, as for {\it Caltech101} can be observed, where SLR achieves the best ranking performance compared to state-of-the-art methods.
\begin{table}[!ht]
\centering
\caption{Comparison of mAP($\%$) on {\it $ImageNet50$}.}\label{tab:2}
\begin{tabular}{|c|c|c|c|c|}
\hline
Euclid            &LMNN         &FRML            &OASIS            &SLR \\\hline
6.8$\pm$0.1     &7.6$\pm$0.1&   8.2$\pm$0.1&11.1$\pm$0.1   &\textbf{14.2$\pm$0.1} \\\hline
\end{tabular}
\end{table}

\begin{table}[!ht]
\centering
\caption{Comparison of CPU Time (minutes).}\label{tab:3}
\begin{tabular}{|l|c|c|c|c|}
\hline
Data            &LMNN   &FRML   &OASIS   &SLR        \\\hline
{\it Caltech101}&725.5 &766.8  &249.5   &\textbf{3.2}   \\\hline
{\it ImageNet50}&521.0 &609.3  &614.6   &\textbf{30.0}   \\\hline
\end{tabular}
\end{table}

To further verify the efficiency of the proposed method, we compare the run time of methods in the experiments shown in Table~\ref{tab:3}. It is observed that the proposed method is more efficient than other methods. We attribute it to the fact that SLR costs $\Ocal(nd^2)$ to go through all examples while the other methods take $\Ocal(Nd^2)$ to learn a good metric, where $N$ is the number of constraints and $N\gg n$.

\section{Conclusions}
\label{sec:conclusion}

In this work, we use bilinear model to rank large-scale images. The challenge of large-scale data is alleviated by writing the learning problem as a matrix regression task, which has a closed-form solution. To further compress data, the randomized variant of the proposed method is developed. We address the challenge from high dimensionality by applying alternating method with low rank assumption. Unlike most of previous methods with low rank assumption, the solution of the proposed method is guaranteed to be global optimal. Different from many previous methods that use stochastic algorithm for efficiency, our method can go through all pairs of images without information lost. In the future, we plan to apply our randomized algorithm to the dataset with the huge number of images that cannot be handled by a single machine. We will also explore our method in other applications (e.g., high dimensional document data).

\bibliographystyle{aaai}
\bibliography{mr_15}

\end{document}